\documentclass[letterpaper, 10 pt, conference]{ieeeconf}  

\IEEEoverridecommandlockouts                        

\usepackage{amsmath}
\usepackage{amssymb}
\usepackage{amsthm}
\usepackage{gensymb}
\usepackage{times}
\usepackage{helvet}
\usepackage{courier}
\usepackage{graphicx}
\usepackage{caption}
\usepackage{subcaption}
\usepackage{multirow}
\usepackage{cite}

\theoremstyle{plain}
\newtheorem{thm}{Theorem} 

\theoremstyle{plain}
\newtheorem{lemma}{Lemma} 

\theoremstyle{plain}

\theoremstyle{plain}

\theoremstyle{definition}
\newtheorem{defn}{Definition} 

\theoremstyle{definition}

\title{\LARGE \bf Visual Generalized Coordinates $^{*}$}

\author{
    M Seetha Ramaiah$^{1}$ \and Amitabha Mukerjee$^{2}$ \and 
    Arindam Chakraborty$^{3}$ \and Sadbodh Sharma$^{4}$
    \thanks{$^{*}$ This work was supported by the Research-I foundation.}
    \thanks{$^{1}, ^{2}$ Department of Computer Science \& Engineering;
        $^{3}, ^{4}$ Center for Mechatronics; 
        Indian Institute of Technology Kanpur.
        {\small \{msram$^{1}$, amit$^{2}$, arindamc$^{3}$, sadbodh$^{4}$\}@iitk.ac.in}} 
}
\begin{document}

\maketitle
\thispagestyle{empty}
\pagestyle{empty}

\begin{abstract} 
An open problem in robotics is that of using vision to
identify a robot's own body and the world around it.  Many
models attempt to recover the traditional C-space
parameters. Instead, we propose an alternative C-space by
deriving generalized coordinates from $n$ images of the
robot. We show that the space of such images is bijective
to the motion space, so these images lie on a manifold
$\mathcal{V}$ homeomorphic to the canonical C-space. 
We now approximate this manifold as a set of
$n$ neighbourhood tangent spaces that result in a
graph, which we call the Visual Roadmap (VRM).
Given a new robot image, we perform
inverse kinematics visually by interpolating between nearby
images in the image space. Obstacles are projected onto the
VRM in $O(n)$ time by superimposition of images, leading to
the identification of collision poses. The edges joining
the free nodes can now be checked with a visual local
planner, and free-space motions computed in $O(nlogn)$
time. This enables us to plan paths in the image space for
a robot manipulator with unknown link geometries, DOF, 
kinematics, obstacles, and camera pose. We sketch the
proofs for the main theoretical ideas, identify the
assumptions, and demonstrate the approach for both
articulated and mobile robots.	We also investigate the
feasibility of the process by investigating various metrics
and image sampling densities, and demonstrate it on
simulated and real robots.
\end{abstract}

\section{Introduction}
\label{sec:intro}

\begin{figure}[t!]
 \begin{center}
    \begin{subfigure}[b]{0.48\columnwidth}
        \centering
        \includegraphics[width=\columnwidth]{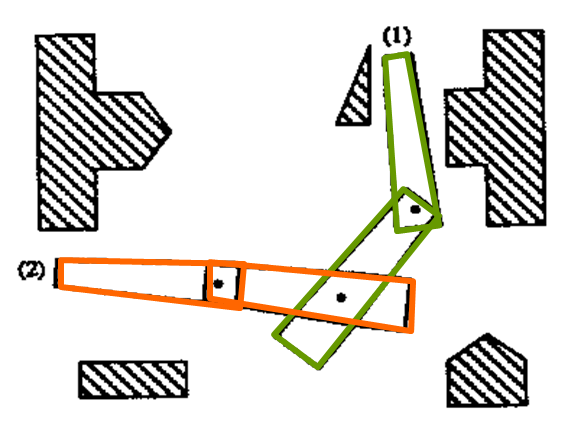}
        \caption{}
        \label{fig:1a}
    \end{subfigure}%
    \hspace{0.1mm}
    \begin{subfigure}[b]{0.48\columnwidth}
        \centering
        \includegraphics[width=\columnwidth]{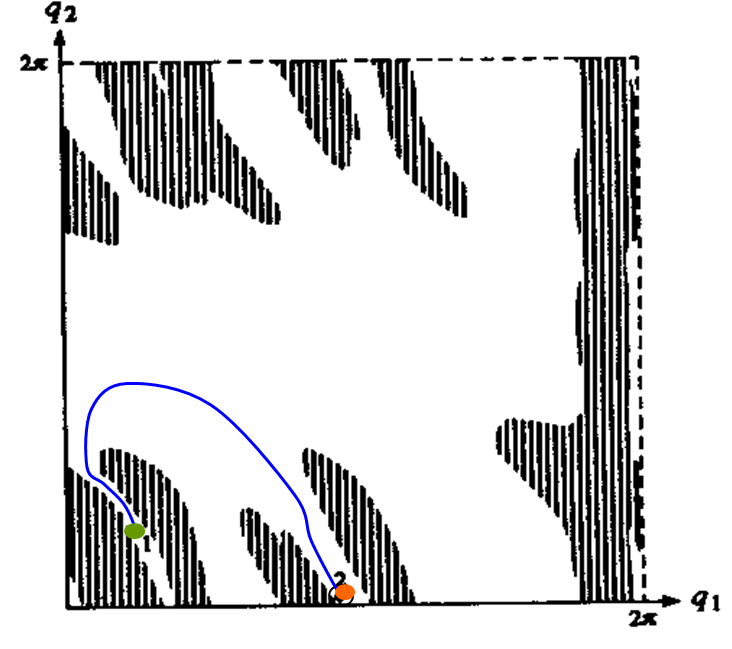}
        \caption{}
        \label{fig:1b}
    \end{subfigure}
    \\
    \begin{subfigure}[b]{0.48\columnwidth}
        \centering
        \includegraphics[width=\columnwidth]{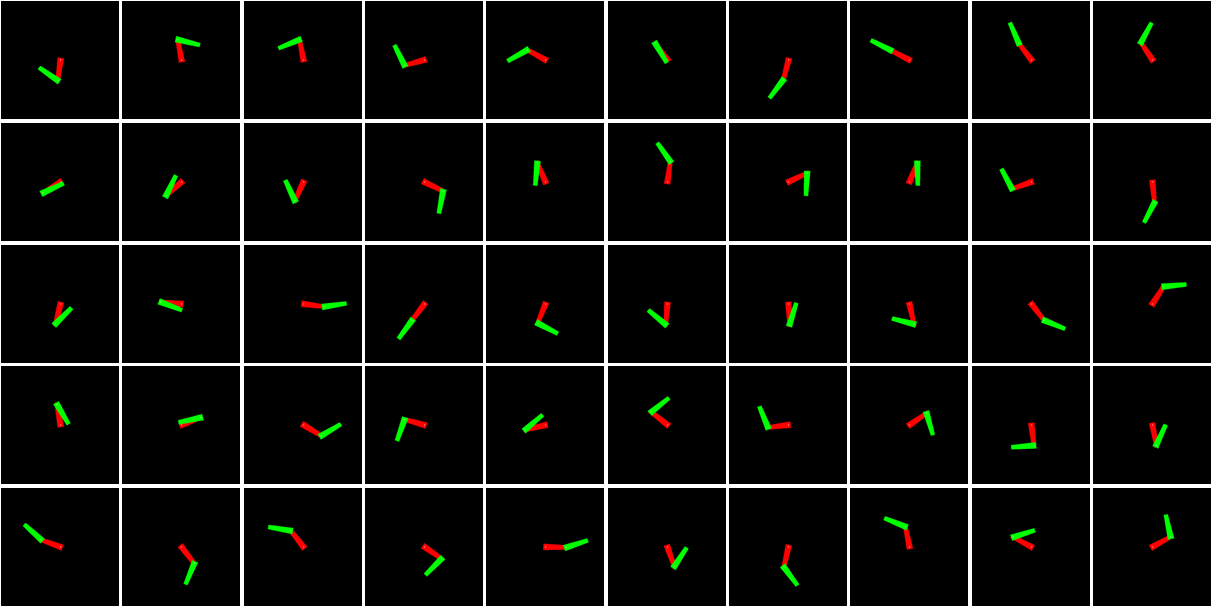}
        \caption{}
        \label{fig:1c}
    \end{subfigure}%
    \hspace{0.1mm}
    \begin{subfigure}[b]{0.48\columnwidth}
        \centering
        \includegraphics[width=\columnwidth]{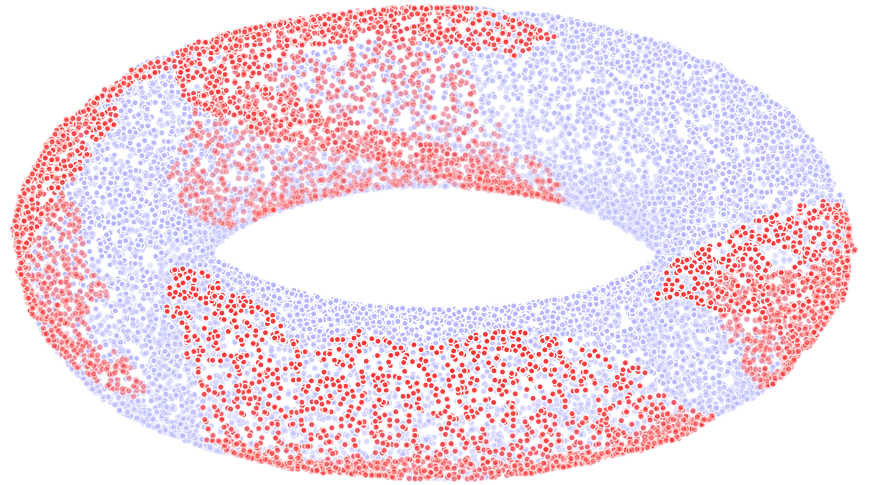}
        \caption{}
        \label{fig:1d}
    \end{subfigure}
    \\
     \begin{subfigure}[b]{0.48\columnwidth}
        \centering
        \includegraphics[width=\columnwidth]{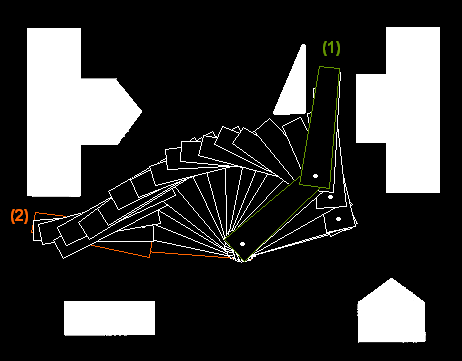}
        \caption{}
        \label{fig:1e}
    \end{subfigure}
    \hspace{0.1mm}
    \begin{subfigure}[b]{0.48\columnwidth}
        \centering
        \includegraphics[width=\columnwidth]{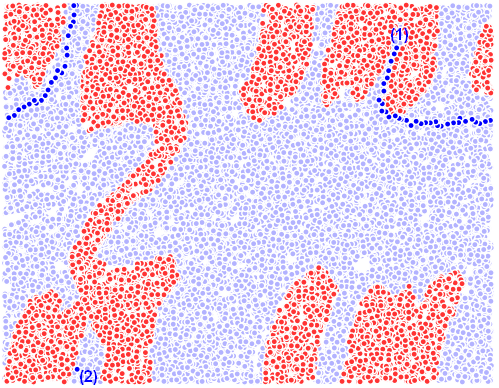}
        \caption{}
        \label{fig:1f}
    \end{subfigure}
    \caption{{\em Overview example}:
        (a) A 2-link planar arm
	      from~\cite{latombe-1996_robot} - 
	      both joints rotate fully around (C-space is a torus).
        (b) the canonical C-space
            from~\cite{latombe-1996_robot} and a path for the poses in (a).
        (c) Fifty of 20000 sample images from a simulation of a similar arm. 
        (d) The image manifold $\mathcal{V}$, visualized here
            in $\mathbb{R}^{3}$ using the Isomap
            algorithm~\cite{tenenbaum-deSilva-langford-00_isomap-manifold-dimensionality-reduction}, shows that the robot images - from a $570\times570$-dimensional
            image space - lie on the 2-D surface of a torus.  Thus
            the {\em Visual Configuration Space} (VCS) has the same structure
            as the canonical C-Space. The Visual Roadmap (VRM)
            is a graph embedded on the VCS, represented as a set of
            tangent spaces (charts). 
            Red points are non-free configurations, identified
            by overlapping background-subtracted robot images with 
            obstacle images. 
        (e) Workspace trajectory between the two poses of (a),
            found using the VRM.
        (f) Path (blue dots) shown on the cut-open torus.  Note that the VCS
            dimension $q_{2}$ (vertical axis) has flipped $\theta_{2}$
            in (b), and both axes are circularly shifted.
    } 
    \label{fig:rmp-2dof}
  \end{center}
\end{figure}

Humans and animals routinely use prior sensorimotor experience to build
motor models, and use vision for gross motor tasks in novel environments. 
Achieving similar abilities, without having to calibrate a robot's own body
structure, or estimate exact 3-D positions,  is a touchstone problem
for robotics 
(e.g. see~\cite{engelberger-Joseph-1980_robotics-in-practice} ch.9).   Such an
approach would enable a robot to work in less controlled environments, as is
being increasingly demanded in social and interactive applications for
robots. 

There have been two methods for approaching this problem - either based on
learning a {\em body schema}~\cite{poincare-1895-space-and-geometry,
hoffmann-marques-10_body-schema-in-robotics_review,
pierce-kuipers-97_map-learning-uninterpreted-sensors,
philipona-oregan-2003_perception-of-sturcture-unknown-sensors,
arleo-smeraldi-04_cognitive-navigation_nonuniform-gabor_reinforcement,
stober-fishgold-kuipers-09_sensor-map-discovery}, or by
fitting a canonical robot
model~\cite{sigaud-sala-11_online-regression-learning-robot-models_survey}.
Body schema approaches have not scaled up to full scale robotic models
or used for  global motion planning, and
robot model regression requires intrusive structures on the robot
~\cite{sturm-11_approaches-to-probabilistic-model-learning-for-mobile-manipulation-robots}
and even then it cannot sense the environment. 

Another approach, {\em visual servoing} attempts to estimate the 
motion needed for small changes in image features. 
However, visual servoing models cannot construct models spanning large
changes in robot pose, since the pseudo-inverse
of the image Jacobian can be computed only over small motions.  Recently,
global motion planning algorithms have been proposed by stitching together
local visual servos~\cite{kazemi-gupta-10_path-planning-visual-servoing}, but
these require that the goal be constantly visible.

\subsection{Visual Generalized Coordinates}
The notion of Configuration Space is fundamental to conceptualizing
multi-body motion.  The configuration of a
system with $d$ degrees of freedom can usually be
specified in terms of $d$ independent parameters, known as {\em generalized
coordinates} (GC). 
Thus, for a planar robot arm with two links, as in fig.~\ref{fig:rmp-2dof}a,
the canonical choice for GC is to use 
the joint angles $(\theta_1, \theta_2)$.  However, this is
only one of many (potentially infinite) choices of coordinates, 
each resulting in a different C-space.  GCs
need not specify joint angles or any
motion parameter - they just need to uniquely specify the pose.
One of our main aims is to show that an alternate GC 
can be learned from the robot's appearance alone, i.e. from a set of images.
These visual coordinates are homeomorphic to
the canonical coordinates - as in fig.~\ref{fig:rmp-2dof}d, where we note
that the image manifold (the {\em Visual Configuration Space}, VCS)
is a torus, just like the canonical  $(\theta_1,
\theta_2)$ manifold.   This is particularly notable since the image
dimensionality $\approx3\times10^5$, so the image space is enormous; 
yet the images that can show robot poses lie on this tiny 
two-dimensional subspace.  This can be explained by noting that
the probability of a random image being a robot image is vanishingly small. 
The {\em Visual Manifold theorem} below formalizes these claims. 

Such a system would have many advantages.  For example, if an unknown robot
is given to us, we can learn its VCS by observing a
random set of poses as it moves.  In fact, 
this idea draws inspiration from proposals for how an infant learns
to use its limbs~\cite{von-Hofsten-04trics_action-perspective-on-motor-development,adolph-berger-06_motor-development}.
Now if a task is specified visually - 
such as an object to grasp - the pose desired for executing it may be easier
to specify in terms of its image, or that of its gripper (the set of gripper
images also form an equivalent manifold). 
Given a novel goal image, the system can interpolate between nearby
known poses to reach the desired pose (inverse kinematics). Assuming we have a
controller that can repeat previously seen poses, the robot can now reach for the
object by traversing a set of landmark images. 
Obstacles introduced now can be superimposed on the
set of images to identify collision configurations. 
Given a set of possibly multiple cameras, this enables
the system to find paths avoiding obstacles. 
Parts of its body or workspace that is accessed repeatedly can have
a finer model (by sampling more images from this part of the workspace). 
The system can cluster task trajectories to learn action
schemas etc.   
One could then ``imagine'' the consequence of a motor command, and compare
these quickly to find discrepancies~\cite{stening-jacobsson-05_imagination-of-sensorimotor-robot-flow}.  
All this is done without any  
knowledge of robot kinematics or shape, its environment, or
even the camera poses.  

%
%
The model proposed here has some constraints. 
It requires that the set of cameras be
able to see the robot in all its poses.  Thus, it 
is more suited for articulated
robot arms, though it would also work for a mobile robot seen from a roof
camera.
Also, it requires that every pose of the robot
must be visually distinguishable - i.e. different poses should look different
from at least one of the camera views.

Our main contributions are a) to show that
configuration spaces based on robot images, and not joint angles, exist;
b) present a sampling-based algorithm that
takes a large set of robot pose images, and 
constructs piecewise approximations
in terms of local neighbourhoods on the image manifold. 
c) demonstrate how such a visual C-space can be equivalently used
to find a roadmap and identify poses (inverse kinematics), and plan motions. 

%

Section~\ref{sec:flowchart} gives an overview of the algorithm. 
Section~\ref{sec:vcs} presents a
theoretical analysis proving the existence of the VCS 
(Visual Manifold Theorem) and that obstacles in the workspace can be modelled
via image superposition (Visual Collision Theorem).  
Section~\ref{sec:VRM-MP} constructs a discrete model of the image manifold
$\mathcal{V}$ by stitching together neighbouring images into a graph
that we call the {\em Visual Roadmap} (VRM).  
This is analogous to roadmaps used in
sampling based motion planning~\cite{latombe-1996_robot,Lav06}. 
During the VRM construction, only immovable parts of the environment are
present; obstacles etc can be introduced later. 
For motion planning purposes, the robot foreground in each image is obtained by
removing the fixed background; these are superimposed on an obstacle image to
identify the collision states. Note that this is equivalent to modelling the
obstacle as the convex hull of the visibility cones 
(Fig.~\ref{fig:multi-view-cameras}).   
Each edge between neighbouring free space nodes is now tested
using one of several visual local planners (section~\ref{sec:local_planner}). 
Now, given images for the Start and Goal poses of the robot, one
can add edges from
these to the nearest safe neighbours in the VRM, and find a path on the
graph.  Section~\ref{sec:empirical_results} presents an empirical analysis of
the various choices for image space metrics and local planners, and
section~\ref{sec:real-robots} shows some demos on real robots. 

\begin{figure}[t!]
\begin{center}
\includegraphics[width=0.9\columnwidth]{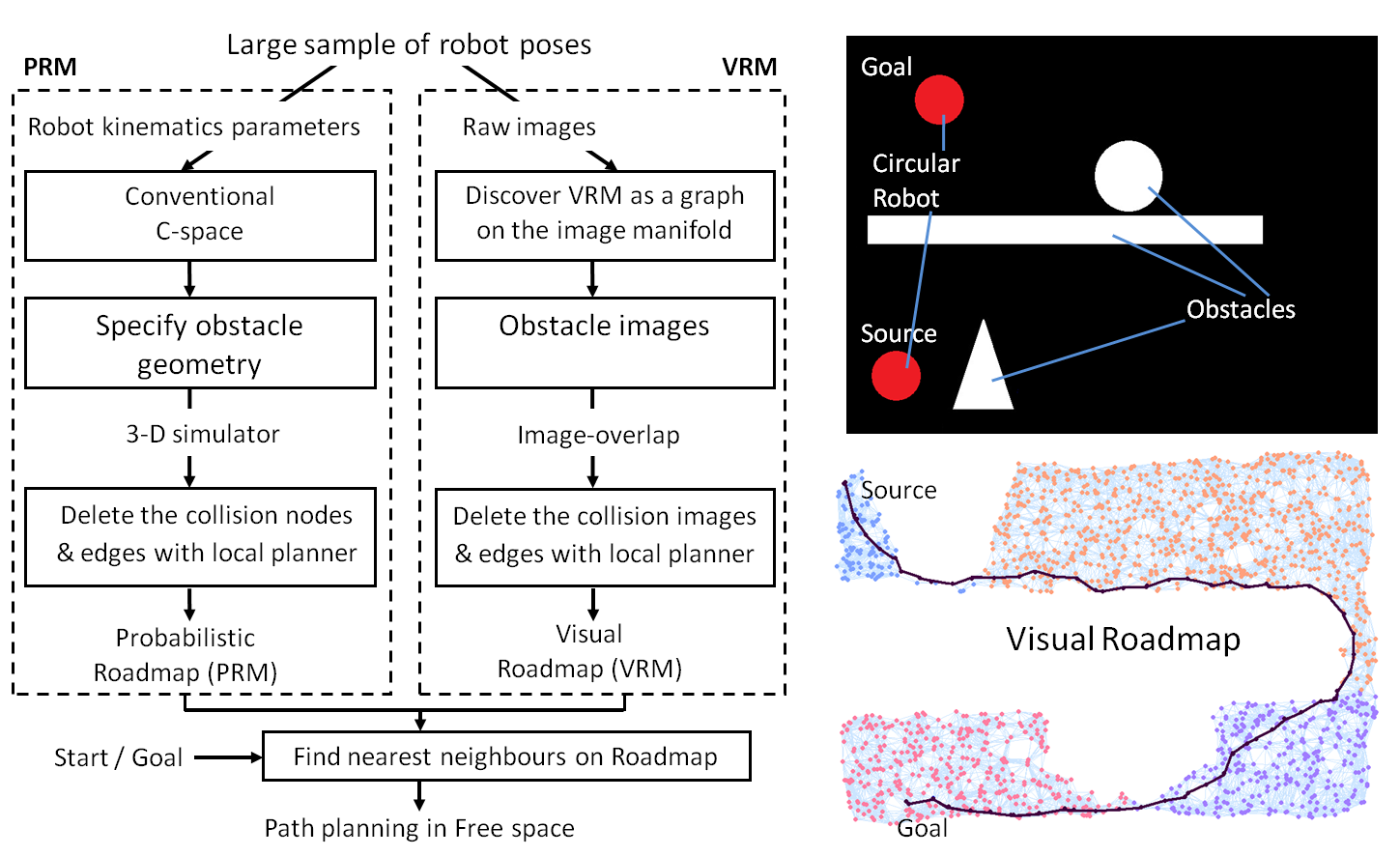}\\
\end{center}
\caption{{\em Visual Roadmap as an analogue of
Probabilistic Roadmap (PRM)}.
In the Visual Roadmap (VRM) approach, a graph is constructed
from the neighbourhoods in image space. 
This requires no
knowledge of robot kinematics or geometry.  Just as with
PRM, one now deletes nodes overlapping the obstacles, and constructs
a path on the remaining edges of the graph. 
The process is illustrated with a simulated mobile robot : the
manifold (bottom right) is constructed solely from a sample of 2000 
images.  The VRM graph is shown with all obstacle nodes removed
and a path identified for a given source and goal.  As in
fig.~\ref{fig:rmp-2dof}f, 
manifold discovery preserves topology but may flip / deform the map. } 
\label{fig:flowchart}
\end{figure}

\section{Algorithm overview and inverse kinematics}
\label{sec:flowchart}
Going from a configuration $q \in \mathcal{Q}$ to the workspace robot shape
and its inverse - known as forward and inverse kinematics - traditionally
involves careful assignment of  
coordinate frames and complex transformations between these.
In the Visual Generalized Coordinates approach, once we have the VRM, inverse
kinematics can be computed for any desired 
pose, presented as an image $x$.  To do this, we find the
images nearest to $x$, and interpolate between these on
the local chart on $\mathcal{V}$.  Now, to plan motions, 
traditional methods require explicit knowledge of obstacle 
geometry as well as a simulator for testing collisions.  Both these are
replaced by visual intersection



Fig.~\ref{fig:flowchart} shows an overview of the VRM algorithm, demonstrated
on a simulated mobile robot.  The idea of the {\em Visual Roadmap} is an
analog to the Probabilistic Roadmap, in that it is a graph $\{V,E\}$ where
$V$ is the set of images sampling the entire workspace, and $E$ the edges
connecting local neighbours.  This is the heart of this work, where the
conventional configuration description $\mathcal{Q}$ (e.g. the joint angle
space) is replaced by a completely different GC
based on images.
The latent
space ${\cal V}$ here is discovered from images, and is homeomorphic to
$\mathcal{Q}$.  

\subsection{Visual Roadmap}
\label{sec:comput}
Discovering Visual Generalized Coordinates requires us to
find the neighbours of an image in image space.  This requires
an image metric - e.g. Fig.~\ref{fig:flowchart} uses a simple euclidean
metric.  
An alternate metric may be to evaluate the swept volume between two poses. 
This can often be effectively approximated by the maximum distance
covered by any point~\cite{kavraki-latombe-overmars-96_PRM-high-dimensional}. 
(see section~\ref{sec:itp}.  Poorer metrics may corrupt the neighbourhood and
call for much denser samples.  Thus, we find that track distance based
metrics, or Hausdorff measures, outperform the Euclidean metrics and require
order of magnitude less samples for the same results~\ref{subsec:metrics}. 

Image space neighbourhoods are then used to construct a local-PCA based
nonlinear 
manifold~\cite{kambhatla-leen-97_dimension-red-by-local-PCA}.  The graph
based on the neighbourhoods is the VRM.  
We observe that there can be situations where multiple robot poses 
look alike to the imaging system (see fig.~\ref{fig:visual-symmetry}).
A critical assumption underlying our approach is that each pose
is distinguishable at least from one camera. 
This is the {\em Visual distinguishability assumption}.  In practice, most
robots already meet this criteria.  

We thus show that the
system can discover a compact non-metric model, that retains the
structure of the conventional Configuration Space $\mathcal{Q}$, but one
that is derived solely based on a dense latent space discovery.  
The discovered lower-dimensional space $\mathcal{V}$ can be
mapped to the robot image space $\mathcal{I}$ and to the
traditional C-space $\mathcal{Q}$.  
We also observe that these 
mappings are the visual analogues for forward  and inverse 
kinematics as in traditional robotics.

\section{Visual  Configuration Space}
\label{sec:vcs}

In order to understand the idea of the Visual Configuration Space,
let us consider the space of images of a robot
$\mathcal{I}$ (e.g. for the 2-DOF robot of 
fig.~\ref{fig:rmp-2dof}).  The input 
Images are high dimensional - if each image is
$640 \times 480$ (approximately $10^5$) pixels, then  
$\mathcal{I} \subset \mathbb{R}^{3 \times 10^{5}}$.
However, given an image $x$ in $\mathcal{I}$, it can be altered as many
ways as the degrees of freedom
$d$, without the resulting image leaving $\mathcal{I}$.  So the
intrinsic dimensionality of $\mathcal{I}$ is 2.
Further, as the links keep rotating, in the end the
image sequence returns to the original image.  Thus the topology of
$\mathcal{I}$ is not 
euclidean ($\mathbb{R}^2$), but a $d$-torus ($S^1 \times S^1$ for the 2-DOF robot).
Also, the image space changes smoothly as it moves, and so the mapping is {\em
  diffeomorphic}.
Thus if  $\mathcal{V}$ is a smooth dense latent space for $\mathcal{I}$, then
every 
image neighbourhood in $\mathcal{I}$ maps to a neighbourhood in $\mathcal{V}$,
and these maps change smoothly as one moves through the poses of the robot.

\begin{figure}[t] \centering
\includegraphics[width=\columnwidth]{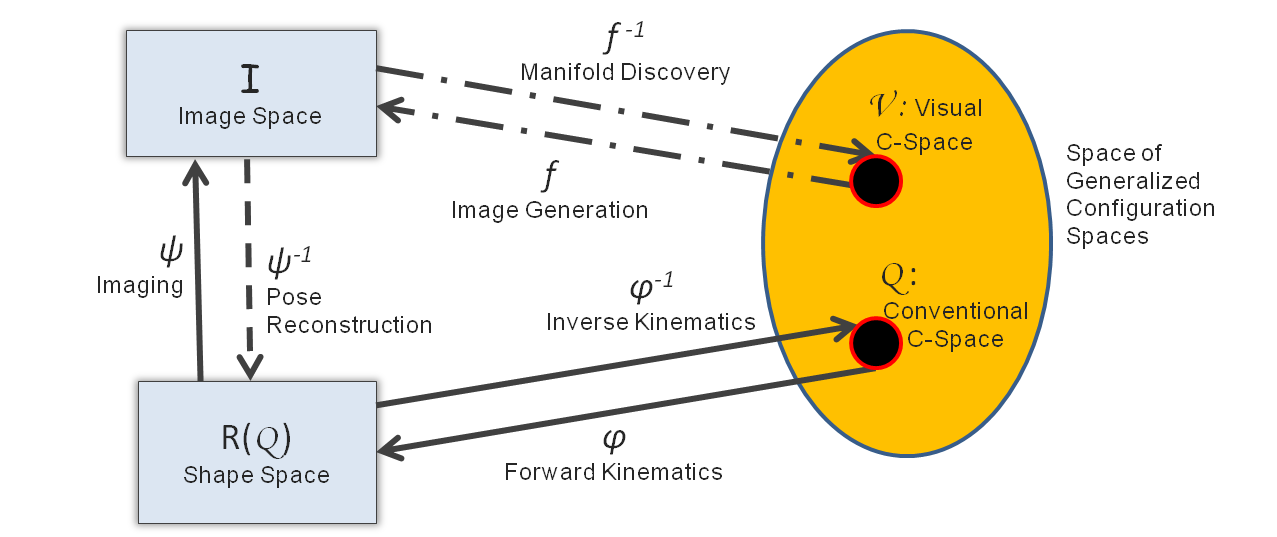}
\caption{ {\em Robot pose, Robot shape, Image and Visual spaces}: 
In order for visual generalized coordinates to exist, 
two robot poses cannot generate the same image.
If this condition holds, we show that any coordinate
for the image manifold constitutes a generalized coordinate system. 
Under such conditions, 
the map $\psi\circ\phi$ between $\mathcal{Q}$, the shape space $R(\mathcal{Q})$
and the image space  $\mathcal{I}$ is bijective, and the image manifold does
not  self-intersect.   The latent space
$\mathcal{V}$ is a specification of generalized coordinates on the image
manifold, and is a member of the collection of C-spaces.
The bijective map $f: \mathcal{V} \leftrightarrow \mathcal{I}$ 
relates robot images to unique points in
$\mathcal{V}$. Given sample images $X \subset \mathcal{I}$, we 
estimate $\mathcal{V}$ via a Local Tangent Space approximation of the
manifold, and do not explicitly compute the mappings
$f$ and $f^{-1}$ (coordinates in $\mathcal{V}$ - shown as dashed lines). 
}
\label{fig:visual-manifold-theorem}
\end{figure}

While these properties hold for the continuous image space, in
practice we work with a representative sample
$X = \{x_1...x_n\} \subset {\cal I}$.
There are a number of manifold
discovery algorithms that one could use (e.g. Isomap~\cite{tenenbaum-deSilva-langford-00_isomap-manifold-dimensionality-reduction}).
However, such methods have difficulty in
introducing new data points and in interpolating local data, so 
we avoid computing the manifold altogether, and restrict ourselves to a
piecewise algorithm, as
in~\cite{kambhatla-leen-97_dimension-red-by-local-PCA,yang-wang-05_better-scaled-local-tangent-space}.

We now establish the conditions under which the space of all images of the
robot would also form a homologous manifold. 

%





\subsection{Visual Distinguishability} 
In general, the imaging transformation $\phi$ is not invertible - i.e. the 3D
positions are not recoverable from the image.  It is only because
the image is being generated under motion constraints, that one can
find a map from the image space to a unique low-dimensional map.  However, this
does not hold in all situations (fig.~\ref{fig:visual-symmetry}); hence we
require that in practice,
there be some colour textures on the robot body, or a restricted range of motion,
that permits distinguishability of all robot poses. 
This is the {\em visual distinguishability} assumption. 

Let $R_q$ be the set of all points of the
workspace occupied by the robot (its volume) in configuration $q$, and 
let $R(\mathcal{Q}) = \{R_q: q \in
\mathcal{Q}\}$ be the set of all robot shapes.  Let $\phi: \mathcal{Q}
\rightarrow R(\mathcal{Q})$ and $\psi: R(\mathcal{Q}) \rightarrow
\mathcal{I}$ be the functions that map a configuration to a shape and
a shape to an image respectively. 
Then the
visual distinguishability assumption requires that the function
$\psi \circ \phi: \mathcal{Q}
\rightarrow \mathcal{I}$ be a bijection as illustrated in
fig.~\ref{fig:visual-manifold-theorem}.

The imaging transformation $\psi\circ\phi$ maps each configuration $q$ to an image
$I_q$ projected by the boundary $\delta R_q$ of shape $R_q$. 
If the visual distinguishability assumption holds, then 
both
$\phi$ and $\phi^{-1}$ exist and are continuous, because small changes in the
robot configuration lead to small changes in its shape and the
corresponding images and vice versa. So, whenever $\mathcal{Q}$ is a
manifold, $\mathcal{I}$ is also a manifold of the same dimension
(i.e., for a $d$ DOF robot the image space is a $d$-dimensional
manifold).  This is the manifold on which the Visual Roadmap (VRM) is
constructed. 

\subsection{Visual Manifold Theorem}
\begin{defn}
A {\em Smoothly Moving Piece-wise Rigid body} (SMPR) is any system 
with a smooth map from its configuration space to its shape space.
\end{defn} 
\begin{lemma}
For a $d$-DOF SMPR, the configuration space is a $d$-dimensional
topological manifold.
\end{lemma}

\begin{defn}
A {\em visually distinguishable system}  is one for which the
visual distinguishability assumption holds.  
\end{defn}

Hence, for a visually distinguishable SMPR, $\psi \circ \phi: \mathcal{Q}
\rightarrow \mathcal{I}$ is a homeomorphism.

\begin{thm}
Whenever $\mathcal{Q}$ is a manifold, $\mathcal{I}$ is a manifold of
the same dimension.
\end{thm}
\begin{proof}
The imaging transformation $\psi \circ \phi$ maps each configuration
$q$ to an image $I_q$ projected by the boundary $\delta R_q$ of shape
$R_q$. If the {\em visual distinguishability} assumption holds, then
both $\phi^{-1}$ and $\psi^{-1}$ exist, and 
the image space $\mathcal{I}$ is homeomorphic
to the configuration space $\mathcal{Q}$.  Hence  $\mathcal{I}$
constitutes a manifold of the
same dimension as that of $\mathcal{Q}$, whenever $\mathcal{Q}$ is a
manifold.
\end{proof}

Fig.~\ref{fig:CRS-joint-manifold} row 1(c) shows a plot of the
residual variance against degrees of freedom for a 2-DOF SCARA arm; the
model is 
clearly captured by a manifold of intrinsic dimensionality two.

\begin{figure}[b]
\centering
\includegraphics[width=0.35\columnwidth]{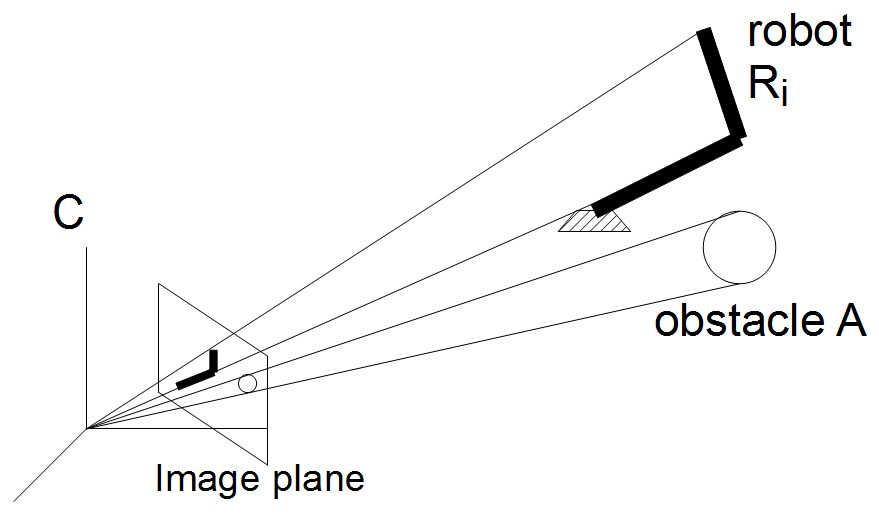}
\includegraphics[width=0.25\columnwidth]{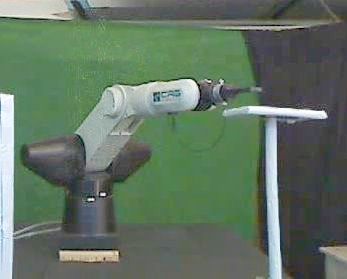}
\includegraphics[width=0.25\columnwidth]{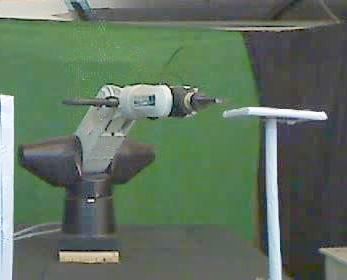}
\caption{
  {\em Imaging the workspace}.  The robot and obstacle lie along the
  projection bundle from the optical center via their image regions 
  in the virtual image plane (left).  If these bundles do not
  intersect, $R \cap A=\emptyset$.  However, the converse is not true.
  Mid \& right:
{\em Images of CRS A465 6-axis robot}  appear to be neighbouring poses,
but close observation reveals that the base joint $\theta_1$ has
rotated by nearly 180 degrees, while $\theta_2$ and $\theta_3$ have
changed sign. Such situations are avoided in the analysis by additional
cameras (e.g. on the gripper), or by adding decals.  Another method for
handling such cases would be to jointly map motor and visual data onto the
same fused manifold }
\label{fig:visual-symmetry}
\end{figure}

\subsection{Difficulties with Manifold discovery algorithms}
\label{sec:non-euclidean}
For robots which involve a motion with an $S^1$ topology,
the C-space and hence the VRM space is not globally Euclidean. For example,
the C-space of a freely-rotating 2-DOF articulated robot is
$S^{1} \times S^{1} = \bf{T}^2$, which is a torus~\cite{choset-05_robot-motion-theory}. Traditional
nonlinear dimensionality reduction (NLDR) algorithms (e.g.~\cite{tenenbaum-deSilva-langford-00_isomap-manifold-dimensionality-reduction})
assume that the target space for dimensionality reduction is a euclidean space (a subspace of
$\mathbb{R}^n$).
This means that a $d$-torus manifold, which is $d$-dimensional,
cannot be globally mapped to an
$\mathbb{R}^d$ space, with which it is locally homeomorphic.   Another
practical difficulty with NLDR algorithms is 
that it is very challenging to add new points to the manifold without
recomputing the entire structure.  

At the same time, the global non-linear coordinate
is little more than a convenience, and does not  materially affect the
modelling, which can be done in a piecewise linear manner. 
Thus, we avoid computing global coordinates altogether, and 
use the local neighbourhood graphs for planning global
paths and local tangent spaces, discovered using Principal Component 
Analysis (PCA), for checking the
safety of edges (local planner). These local tangent spaces, in
theory, correspond to charts which when stitched together form an
atlas for the image manifold.



We next describe how obstacles are mapped on the VCS for collision 
detection.


\begin{figure}[th]
\begin{center}
\includegraphics[width=0.6\columnwidth]{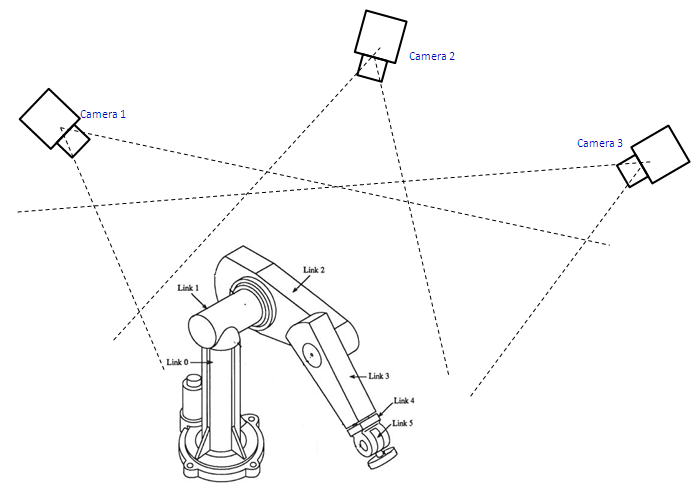}
\end{center}
\caption{{\em Conservative modelling of 3-D obstacles}. 
For 3-D obstacles, the robot image must not overlap with the obstacle in at
least one camera view.
If some part of the robot occludes the obstacle
in the cones for {\em all} the cameras 
the system will consider it to be a collision situation.  
  } 
\label{fig:multi-view-cameras}
\end{figure}

\subsection {Collision Detection in VCS}
\label{sec:collision}
In the imaging process, robot and  obstacle are mapped to a bundle of rays 
converging on the camera optical center (figure~\ref{fig:visual-symmetry}).   
 
Let $^{C}R_i$ be the bundle subtended at camera optical center
$^CO$ by the robot in configuration $q^{(i)}$, $^{C}A$ be the bundle
subtended at $^CO$ by the obstacle $A$ and $^{I}R_i$, $^{I}A$ be
the image regions corresponding to the robot and the obstacle.

\begin{lemma}
If $^{C}R_i \cap ^CA = \emptyset$ then $A \cap R(q^{(i)}) = \emptyset$.
\end{lemma}

Thus, robot configurations for which the bundles do not intersect with the obstacle
bundle are guaranteed to be in the free space ${\mathcal F}$. Note that the
converse is not true.

\begin{lemma}
$^CA \cap ^CR = \emptyset$ iff $^IA \cap ^IR = \emptyset$. 
\end{lemma}

\begin{thm}
{\em (Visual Collision Theorem)}
For a robot in a given pose $q^{(i)}$, if  ${^I}R_i \cap {^I}A = \emptyset$, then
$q^{(i)} \in {\mathcal F}$.
\end{thm}


We note that the above is a necessary condition, but it is often rather
conservative.  Indeed, the
inverse condition defines {\em occlusion} situations:
where $A \cap R = \emptyset$ but $^CR \cap ^CA$ is non-null. 
This limitation is a result of the information loss in the imaging process.
These can be cause particular difficulties for articulated arms.  In such
cases, one may use multiple cameras; since the Visual Collision Theorem holds 
for all cameras, we may define any space as free if   $^{C}R \cap ^{C}A = \emptyset$ in at least one view.  In this situation, both robot and obstacle
are less conservatively modelled as the intersection of multiple cones.  



In general, for non-orthographic projections, the
higher the ratio of camera distance/focal length, the tighter the bound. 
(e.g the Scara robot arm in section~\ref{sec:scara}). 

\section{Visual Roadmap and Motion Planning Algorithms}
\label{sec:VRM-MP}
A colour image sample $X \in \mathbb{R}^{p \times n}$
where each $r \times c \times 3$ RGB
image is represented as a $p$-dimensional vector $(p = 3rc)$ of 
intensities. We assume that the images are captured
against a fixed background, which can be eliminated,
so that in the foreground images, a pixel is
non-zero if and only if it belongs to the robot. Image $x_i \in X$ 
corresponds to configuration $q_i \in Q$. Let $d(x_i,
x_j)$ be a suitable metric (e.g.  Euclidean distance between the image
vectors), and let $\mathcal{N}(x)$ be the set of $k$-nearest neighbours of
$x$. 

Next, we construct a graph $G(V, E)$ over the $n$ nodes so that $v_i \in
V$ corresponds to image $x_i$ (or configuration $q_i$). We add an edge
between two nodes $v_i$ and $v_j$ if 
either $x_i \in \mathcal{N}(x_j)$ or $x_j \in \mathcal{N}(x_i)$ and
assign edge weight $d(x_i, x_j)$. We call this graph the {\em Visual
Roadmap (VRM)}.


\subsection{VRM with Static Obstacles}
For handling obstacles, we
take the background subtracted images, and
test this intersection with the obstacle
image.  Non-empty overlaps imply that the configuration is not free and
we remove the corresponding node and its incident edges from $G$.
If $b \in \mathbb{R}^p$ is the obstacle image vector, then the set
of nodes to be removed from $G$ is 
$V_{collision} = \{v_i: x_i * b \ne \bf{0}\}$, where $*$ denotes entry-wise
product (Hadamard product)
and $\bf{0}$ is the zero-vector. Thus, we obtain a modified graph 
$G'(V', E')$ in which every node represents a free configuration.
However, the edges may still touch some part of the obstacle
in an intermediate pose. Guaranteeing edge-safety is the problem of local 
planner below.  
Note that this process applies to any number of static obstacles. 
%
%
\subsection{Local Planner in VRM}
\label{sec:local_planner}
We say that an edge $(u, v)$ of $G$ is \emph{safe}, if the geodesic from
$u$ to $v$ on the configuration manifold does not contain any image
that overalps with an
obstacle.  The geodesic is approximated by the shortest path on $G$. 
Given that the nodes $V$ are in the free space, we need to guarantee that every
edge is also safe. We describe three local planners that work with robot images
and can be used on visual roadmaps. To make sure that an edge is safe, 
these methods construct a new image that in estimates the swept
volume of the robot in the workspace and check this image for collision. 
Figure~\ref{fig:local_planner_interpolation} shows examples of images
generated by these local planners. 

\begin{figure}[th]
\begin{center}
\includegraphics[width=0.32\columnwidth]{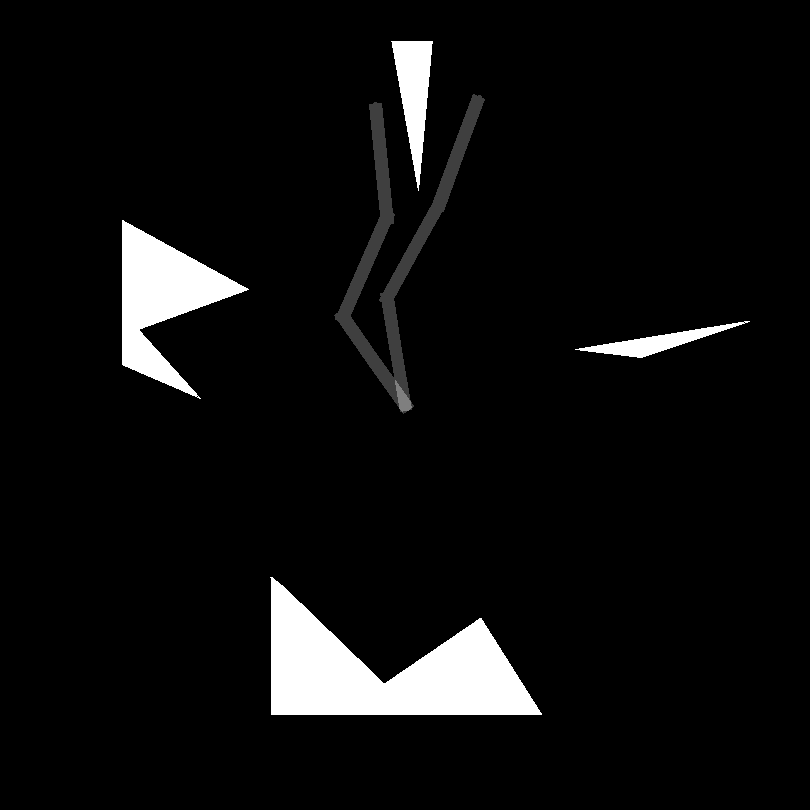}
\includegraphics[width=0.32\columnwidth]{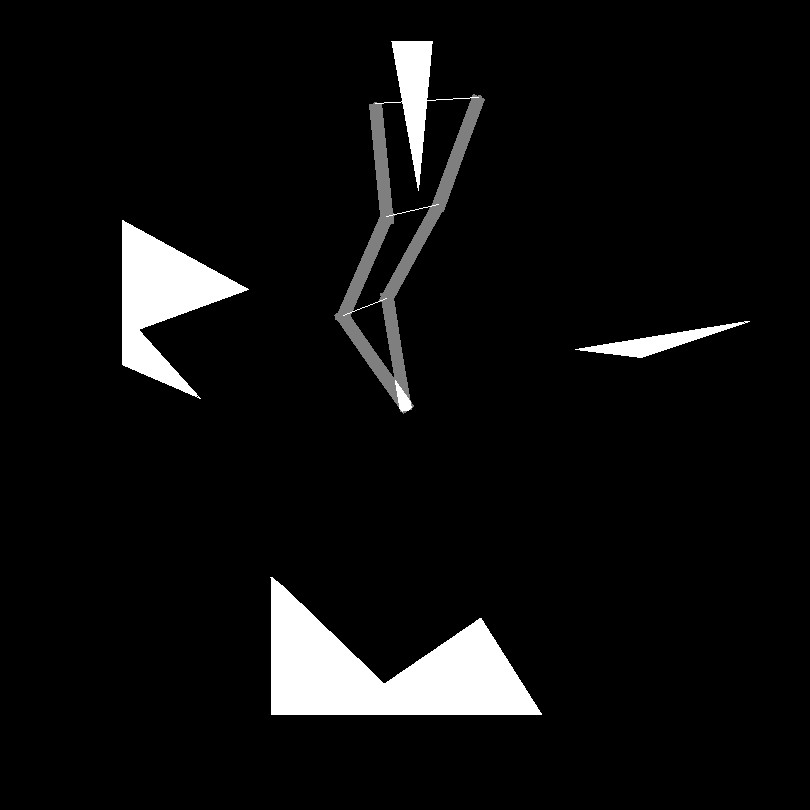}
\includegraphics[width=0.32\columnwidth]{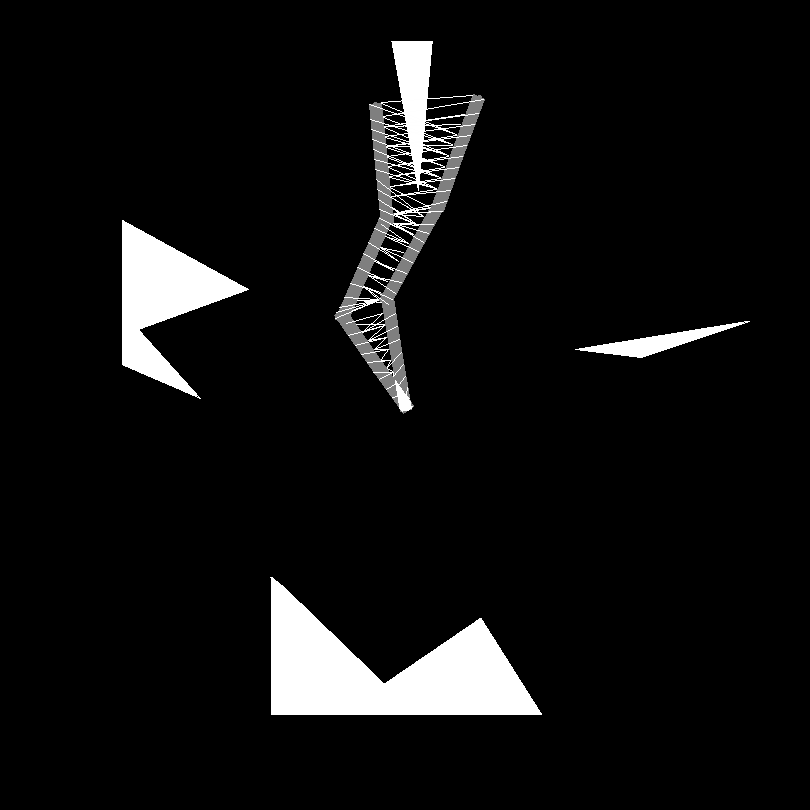}
\end{center}
\caption{{\em Interpolation in each local planner}. 
Local planner using (a)  interpolation on Local Tangent Space (LTS) (b) Ideal 
  Tracked-Points (ITP) (c) Joins of Nearest Shi-Tomasi features
  link-wise (JNST).
  } 
\label{fig:local_planner_interpolation}
\end{figure}

\subsubsection{Interpolation on the Local Tangent Space (LTS)}
\label{pca_lp}
 
For each edge $(u, v) \in E$, let $X^{(u, v)} =
\{x_q: q \in \mathcal{N}(u) \cap \mathcal{N}(v)\}$ be the ${p \times
m}$ matrix of images corresponding to the intersection of neighbours of $u$
and neighbours of $v$ (including $u$ and $v$), where $m$ is the cardinality
of $X^{(u, v)}$.
To see if $(u, v)$ is \emph{safe}, we interpolate the intermediate images on
the tangent space spanned by $X^{(u, v)}$, obtained using PCA.
The target dimension is the number of degrees of freedom $d$, and PCA
maps $X^{(u, v)}$ to a $Y^{(u, v)}$ ($d \times m$). In addition to $Y^{(u,
  v)}$, PCA also gives a ${p \times d}$ orthonormal matrix $W^{(u, v)}$ such
that $X^{(u, v)} = W^{(u, v)}Y^{(u, v)}$ or $Y^{(u, v)} = W^{(u, v)^T}X^{(u,
  v)}$. 
We then interpolate between $y_u$ and $y_v$ to construct $y^{(\alpha)}
=  \alpha*y_u + (1-\alpha)*y_v$ for various values of $\alpha \in (0,
1)$. For each $\alpha$, the image $x^{(\alpha)} = Wy^{(\alpha)}$ must
be in free space.
In practice, the resulting image is a poor interpolation, and rejects
many valid edges; however, the probability of an edge being unsafe
after being passed by the local planner is low (i.e. it is conservative).

The image obtained by a linear interpolation on the local tangent space (LTS)
is a
weighted 
sum of the images in $X^{(u, v)}$. Thus, for collision detection purposes, it
is sufficient to  look at the superimposition of images in $X^{(u, v)}$. This
achieves the same effect as the PCA based method described above, but avoids
the PCA computation.



\subsubsection{Ideal Tracked Points (ITP)}
\label{sec:itp}
Here we assume that a set of points on the robot body can be tracked in all
poses (including occlusions).  Then to see if an edge $(u, v)$ is safe, we
join each
pair of corresponding tracked-points to create
a new image.  This image (i.e. the set of trajectories of the tracked points)
are used for collision detection.  

\subsubsection{Join of Nearest Shi-Tomasi features (JNST)}
In practice, occlusion precludes the tracking of any set of points on the
robot body.  Here we propose an approximation based on high-contrast
points known as the Shi-Tomasi
features~\cite{shi-tomasi_1994_cvpr_good-features-to-track}.
We assume that each link of the robot can be separated and that the
Shi-Tomasi features are computed on each link.
Here we do not know the correspondences between points
in the two images.
The Join of Nearest Shi-Tomasi features approach (JNST) involves
associating each feature point on each
link in $u$ with the nearest feature point in the corresponding link in
$v$.  We do the same in both directions and insert the
joins on the image.

\subsection{Start and Goal states}

For motion planning on the VRM, we need to map 
the source ($s$) and target ($t$) images onto the
VRM $G$. We first ensure that the poses $s,t$ themselves are
in free space.  We then add  these to $G$ and connect them with their
$k$-nearest 
neighbours in $X$. We then run a local planner
on the new edges and find the shortest path
between $s$ and $t$ as before. Adding a new node (image) to the graph
is a computation that requires $O(nk)$ distance computation steps
for finding.
Time for distance calculation depends on the metric used.
This approach again, is almost identical to traditional roadmap
methods~\cite{choset-05_robot-motion-theory}, except that the tests are all
visual. 

\section{Empirical analysis : Metrics and Local Planners}
 \label{sec:empirical_results}
Factors affecting the quality of paths in VRM
include sampling density, the metric used, and the local
planner. We now present an empirical study of these
aspects (fig~\ref{fig:lp_plots}) on a planar 3-link 
simulated arm and a set of obstacles similar to those in
figure~\ref{fig:local_planner_interpolation}. 

\subsection{Gold Standard Local Planner}
In the traditional Configuration Space, two configurations
are assumed to be joined by a linear join between them. 
To see if an edge $(u, v)$ is actually safe, we generate
intermediate pose images by interpolating joint
angle vectors at an $\epsilon$ resolution.
We observe that a linear interpolation in joint angle space
need not be the same as an interpoloation on visual C-space,
but we assume the difference would be fairly small for a reasonable
sampling density.  If all these images are
collision free, we treat $(u, v)$ to be safe.
The performance of
the local planners is evaluated relative
to this gold standard local planner.
Results reported here use $\epsilon = 1\degree$.



\begin{figure}[t]
\centering
    \begin{subfigure}[t]{0.48\columnwidth}%
        \includegraphics[width=\columnwidth]{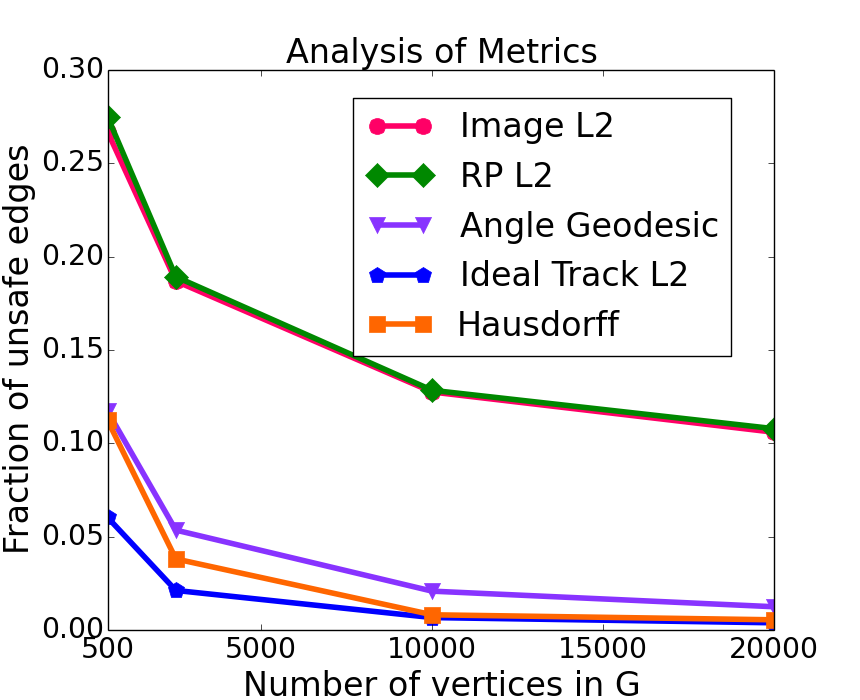}%
        \caption{}%
    \end{subfigure}%
    \begin{subfigure}[t]{0.48\columnwidth}%
        \includegraphics[width=\columnwidth]{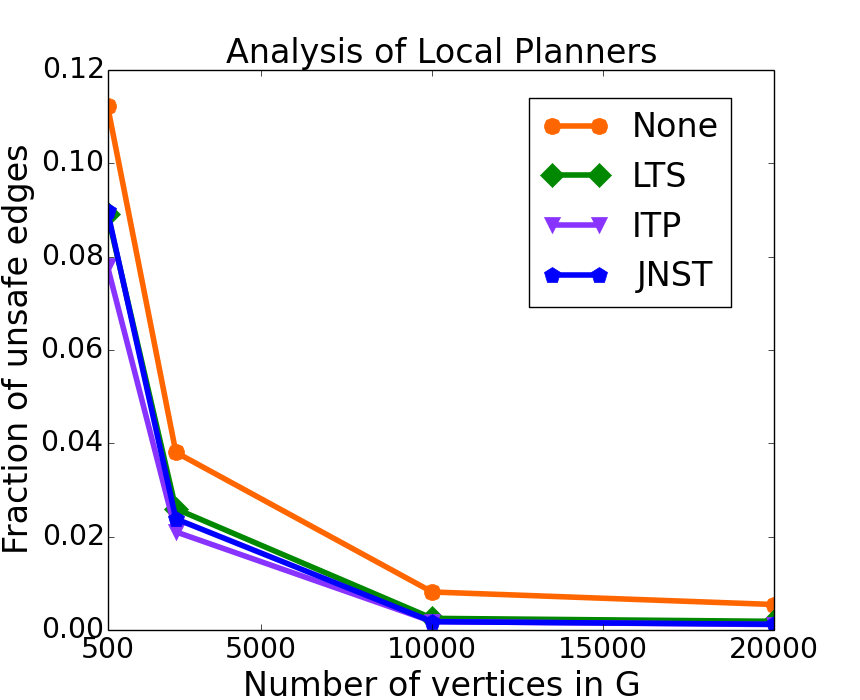}%
        \caption{}%
    \end{subfigure}%
    \caption{{\em Empirical analysis of metrics and local planners}.
a) Edge failures without a local planner.  b) Local planner performance plots
based on Hausdorff metric.   The JNST local planner performs almost as well
as the ITP local planner, which is not implementable in practice. 
}
\label{fig:lp_plots}
\end{figure}

\subsection{Effect of Sampling Density and Distance Metric}
\label{subsec:metrics}

Plots in figure~\ref{fig:lp_plots}a clearly suggest that the sampling density (i.e., the number of images used to construct the visual roadmap) heavily affects the fraction of unsafe edges and hence the quality of paths. The more dense the sample is, the better the paths.

We present the effect of several representations of the configuration space along with appropriate distance metric for each case. Table~\ref{tab:metrics} lists the different representations and the corresponding distance metric used to compute neighbourhoods.

\begin{table}[h]
    \centering
    \caption{
        Different representations of the configuration space and the distance metric used with each    
        representation. Short forms mentioned here have been used in table~\ref{tab:metrics_n_lp}.
    }
    \begin{tabular}{| l | l | l | l |}
    \hline
    Representation & Distance Metric & Short Form \\
    \hline
    1. Raw RGB images of the robot & $L_2$  &  Img $L_2$ \\
    2. Random projections of images & $L_2$ & RP $L_2$ \\
    3. Joint angle vector of the robot & Geodesic & $\theta$-G \\
    4. Ideal tracked points & $L_2$ & ITP $L_2$ \\
    5. Shi-Tomasi features link-wise & Hausdorff & ST-H\\
    \hline
    \end{tabular}
    \label{tab:metrics}
\end{table}

To find the distance between two images we just flatten all the channels of each image into a single vector and use the standard Euclidean  ($L_2$) distance on the resulting vectors. In our experiments we used 30,000 (100x100x3) dimensional vectors for image distance.

Random projections (RP~\cite{bingham2001random,dasgupta2000experiments}) is a dimensionality reduction method that preserves $L_2$ distances. In our experiments we projected the 30,000 dimensional image vectors onto 2000 Gaussian random unit vectors to obtain a 2000 dimensional representation of each image. The experiments show that the $L_2$ distance of RP vectors does almost as well as that on the image vectors. Since the distance computation is done on much smaller vectors, the graph construction gets much faster while preserving the neighbourhoods.

Distance between two joint angle vectors is computed as the sum of shortest circular-distance (i.e., treating 0 and $2\pi$ to be the same angle) between individual components. This is in some sense the geodesic distance between the two vectors.

The ideal tracked point (ITP) $L_2$ distance between two configurations is computed as the $L_2$ distance between the vectors obtained by concatenating all the tracked point coordinates of each configuration.

Finally, the Hausdorff distance between two configurations is computed as the sum of Hausdorff distances between the sets of Shi-Tomasi feature points on the corresponding links for the two configurations. Given two sets $A$ and $B$, Hausdorff distance is defined as 

$$d_H(A, B) = max\{ \sup_{a \in A} \inf_{b \in B} d(a, b), \sup_{b \in B} \inf_{a \in A} d(a, b) \}.$$


\begin{table}[h]
    \centering
    \caption{
        Percentage of bad edges remaining after pruning the VRM using each local planner on a graph with 20000 nodes with different metrics. See table~\ref{tab:metrics} for an explanation of these metric spaces.
    }
    \begin{tabular}{| c | l | l | l | l | l |}
    \hline
    \multirow{2}{*}{Local Planner} & \multicolumn{5}{c|}{Metric Space} \\
    \cline{2-6}
    & Img $L_2$ & RP $L_2$ & $\theta$-G & ITP $L_2$ & ST-H \\
    \hline
    None & 10.59 & 10.79 &  1.25 &  0.39 & 0.55 \\
    LTS & 9.18 &  9.34 &  0.43 &  0.09 & 0.19 \\
    ITP & 7.97 &  8.11 &  0.17 &  0.11 & 0.12 \\
    JNST & 9.58 &  9.74 &  0.16 &  0.12 & 0.12 \\
    \hline
    \end{tabular}
    \label{tab:metrics_n_lp}
\end{table}

As can be seen from figure~\ref{fig:lp_plots}b and table~\ref{tab:metrics_n_lp}, JNST local planner performs almost as well as ITP local planner.

\section {Demonstrations on Real Robots}
\label{sec:real-robots}
\begin{figure}[t]
\begin{centering}
\begin{minipage}{\columnwidth}
\includegraphics[width=0.30\columnwidth]{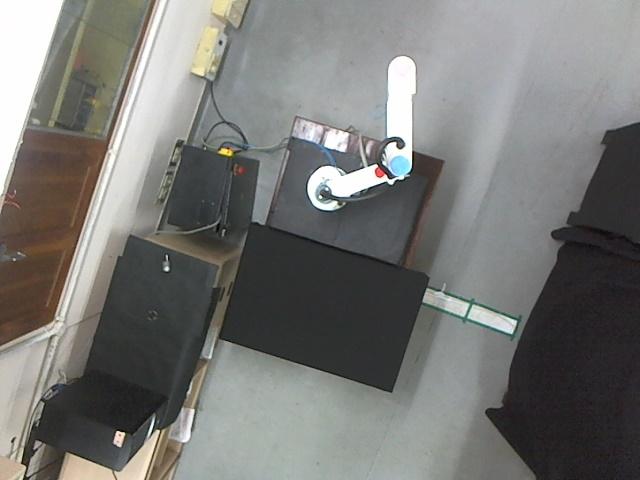}
\includegraphics[width=0.30\columnwidth]{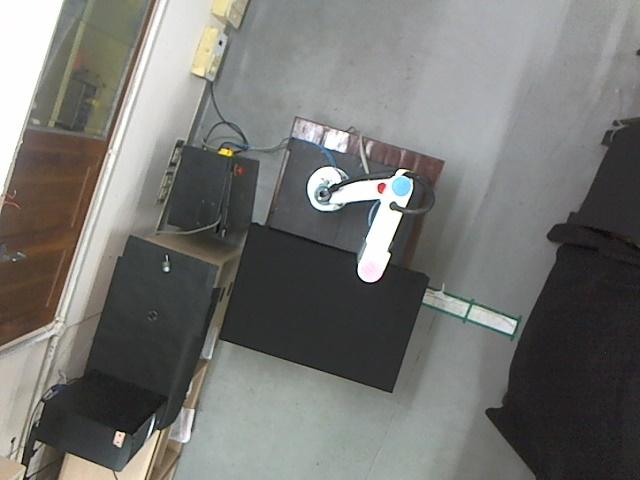}
\includegraphics[width=0.30\columnwidth]{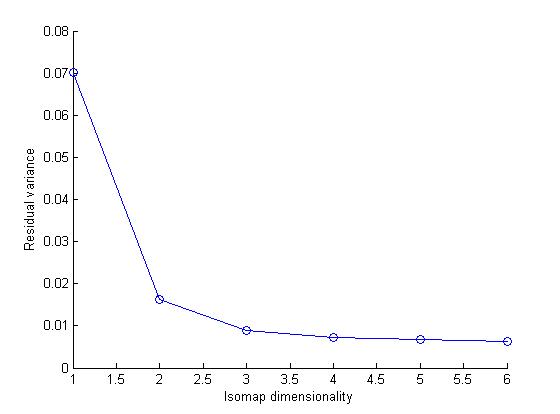}
\end{minipage}

\begin{minipage}{\columnwidth}
\includegraphics[width=0.30\columnwidth]{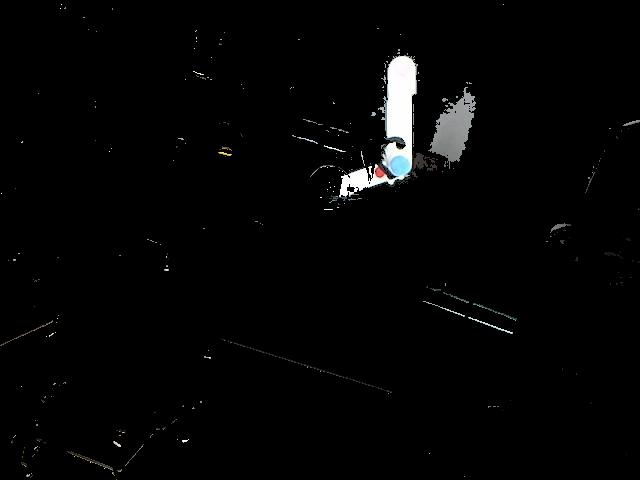}
\includegraphics[width=0.30\columnwidth]{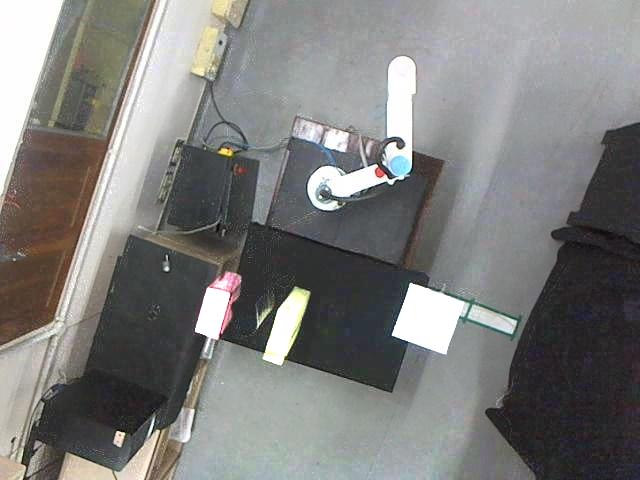}
\includegraphics[height=0.8in,width=0.30\columnwidth]{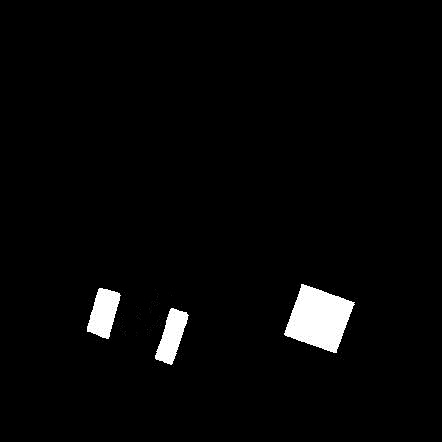}
\end{minipage}

\begin{minipage}{\columnwidth}
\includegraphics[width=.92\columnwidth]{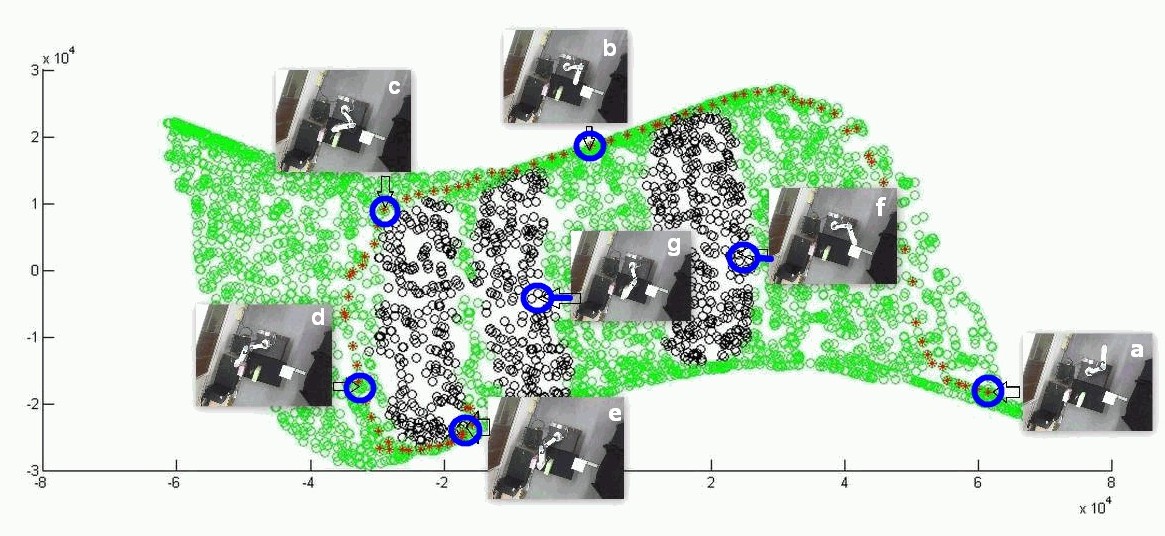}
\end{minipage}

\begin{minipage}{\columnwidth}
\includegraphics[width=0.30\columnwidth]{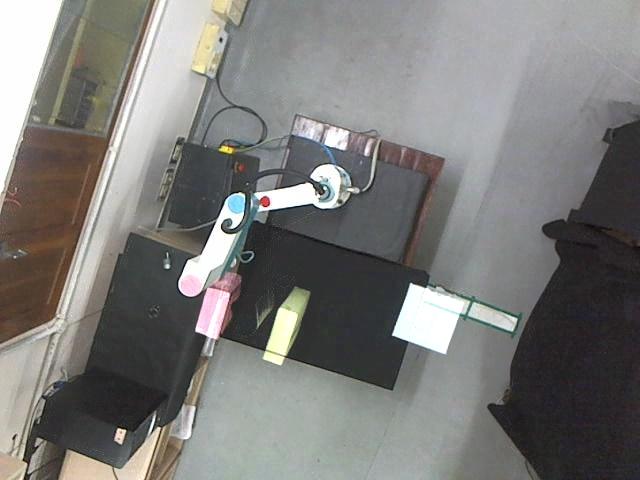}
\includegraphics[width=0.30\columnwidth]{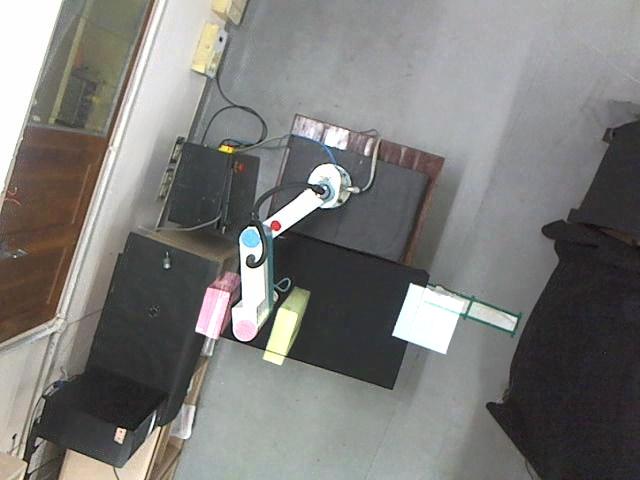}
\includegraphics[width=0.30\columnwidth]{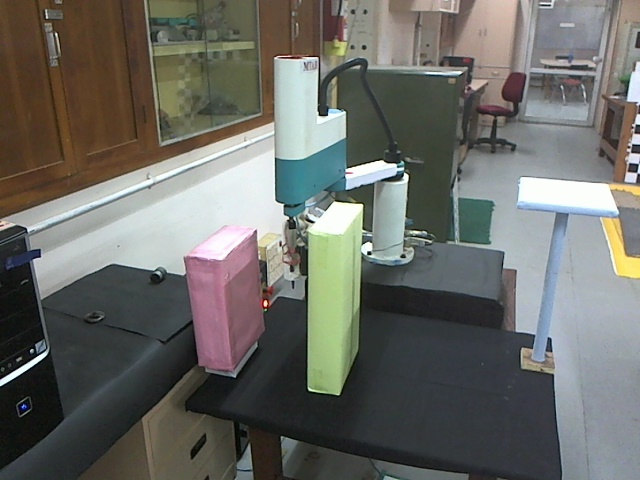}
\end{minipage}

\caption{{\em Path planning for the MTAB Scara robot arm.} Row 1: (a),(b)
 some of
  the 4000 images of the arm.  (c) scree plot. Row 2: incorporating
  obstacles. (a) background subtracted image of the arm, (b) image with
  obstacles. (c) obstacles after image subtraction. 
  Row 3: {\em Visual Configuration Space};  obstacle nodes shown in black, and
  showing a path plotted from start to goal images. Row 4: path being
  executed by Scara. }
\label{fig:scara}
\end{centering}
\end{figure}

\begin{figure}[h]
\begin{center}
\includegraphics[width=\columnwidth]{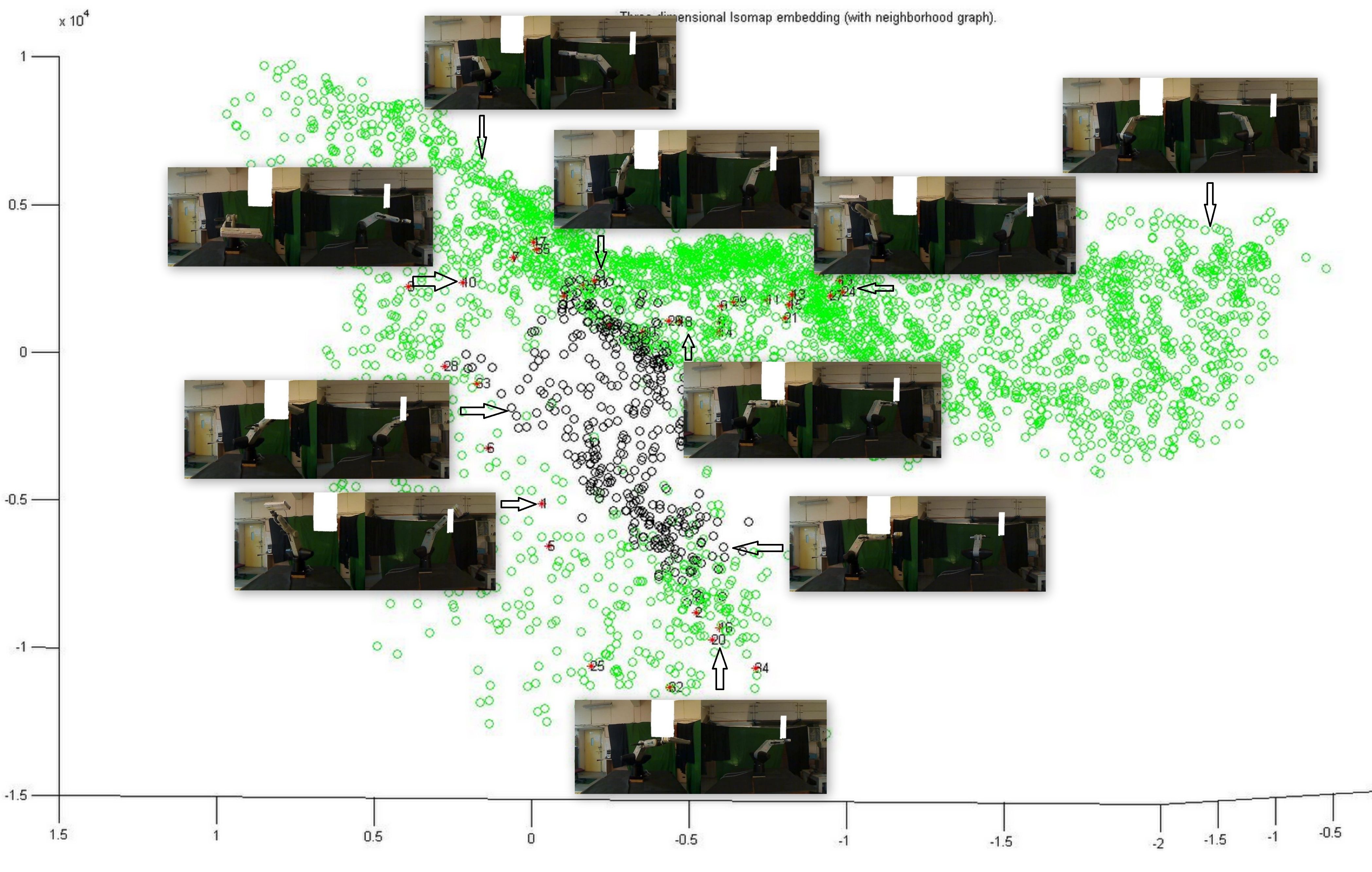}
\caption{{\em Two-camera joint-manifold VRM for a 6-DOF CRS A465 robot.} 
Since this is a 3-D workspace, obstacles cannot be distinguished from a single
view.  Here we use multiple cameras, and the intersection of the cones
provide a (oversized) model for both obstacle and robot.  
To identify potential collision states, 
background-substracted images of the 
obstacle (shown in white here) are overlaid on each foreground robot image
Only if the robot overlaps
the obstacle in all the images, it is a potential collision node.  Collision
nodes shown in black. 
} 
\label{fig:CRS-joint-manifold}
\end{center}
\end{figure}

\subsection{Planar Scara robot}
\label{sec:scara}
We now demonstrate the algorithm for a real robot, a Scara 4 DOF
arm, in which two revolute joints move the first two links in a plane, so the
motion has two degrees of freedom. 
We observe this robot with an overhead camera.  
4000 images are sampled from a video while the robot is moving between random
poses throughout its workspace, and the neighbourhood graph is computed.  
Thereafter, several obstacles are introduced in the workspace 
and the obstacles are discovered via background subtraction.  Note that owing
to the motion being planar, a single camera view is quite adequate. 
A planned path is shown in fig.~\ref{fig:scara}.

\subsection{CRS A465 robot arm}
\label{sec:crs}
Here we have a robot in a 3-D workspace.  Clearly, a single camera view will
not suffice for identifying collision situations.  Hence we construct a
joint manifold of multiple views by stitching the corresponding images
together, and constructing a joint manifold on the combined image space. 
The 3-DOF workspace and a path is shown in
fig.~\ref{fig:CRS-joint-manifold}, with the obstacle nodes marked in black.

\section{Conclusion}
\label{sec:conclusion}

In this work, we have introduced a new approach towards
the longstanding perceptual robotics problem, which 
subsumes the problem of body schema learning~\cite{
poincare-1895-space-and-geometry,hoffmann-marques-10_body-schema-in-robotics_review, philipona-oregan-2003_perception-of-sturcture-unknown-sensors}. 
Although it has been long known that there may be many kinds of generalized
coordinates, so far there have been few attempts in robotics to build on this
intuition.  The proposed paradigm attempts to develop such a non-traditional GC,
and approximates the C-space that results from it in terms of a
neighbourhood graph on the set of images.  We show how such a formulation
for tasks such as inverse kinematic or for motion planning. 

Unlike in methods used in robotics today, the
Visual Generalized Coordinates approach eliminates several expensive
aspects of robot modelling and planning.  First, it does not require a
humans to create models for robot geometry or kinematics.  It does not
require precise obstacle shapes and poses, and does not require to calibrate
the cameras so that this can be done.  There is no need
for a precise simulator to test which poses collide with obstacles and which
do not.  
Even the local planner step, based on tracking image points to
nearby images, results in a more principled approach than is available 
presently. 

Another advantage is for environments that are changing rapidly, 
e.g. in interaction with humans or other robots.  
New obstacles are updated in $O(n)$ time, but small motions
by another agent require $O(m)$, where there are $m$ nodes near
the obstacle boundary. 

The idea of generalized coordinates originated in Lagrangian
dynamics, and here is another direction that needs to be pursued. 
Differentiating the GC would result in generalized velocities
and accelerations and this may give rise to a visual dynamics. 

However, there are some significant trade-offs.  First, the approach is not
complete because the obstacle approximations is conservative, and there may
exist paths which it cannot find.  
We observe that humans also face similar
constraints where the vision is less informative.  Secondly, it is applicable
to situations where the entire C-Space is visible.    
While the algorithms are reasonably efficient in time complexity, the space
costs are higher ($O(np)$), since all landmark images need to be stored. 
Another constraint is the Visual Distinguishability assumption, but this may
not be very serious in practice.

The approach presented is only a beginning for discovering generalized
coordinates from sensorimotor data. One of the key future steps would be
to fuse modalities other than vision into a joint manifold.  Thus, if we 
were to construct a fused visuo-motor manifold, then
even if poses that are
separated in motion space look similar, they would remain
distinguishable.  Similarly touch stimuli could be modelled to predict
the result of motions or in preparation for fine-motor tasks. 
Such a process would also make the model more robust against noise arising in
any single modality. 
On the whole, while the ideas presented seem promising, and open up many
possibilities, much  work remains
to deploy Visual Generalized Coordinates fully in theory and in practice.

\bibliographystyle{IEEEtran}
\bibliography{icra16}

\end{document}